\def\year{2017}\relax
\documentclass[letterpaper]{article}
\usepackage{aaai17}
\usepackage{times}
\usepackage{helvet}
\usepackage{courier}

\usepackage{amsmath}
\usepackage{amsthm}
\usepackage{amssymb}
\usepackage{graphicx}

\newtheorem{definition}{Definition}
\newtheorem{theorem}{Theorem}
\newtheorem{lemma}{Lemma}

\newcommand{\E}{\mathrm{E}}

\newcommand{\argmax}{\mathrm{argmax}}
\newcommand{\argmin}{\mathrm{argmin}}

\begin{document}
%
\title{Universum Prescription: Regularization Using Unlabeled Data}
\author{
Xiang Zhang \qquad Yann LeCun \\
Courant Institute of Mathematical Sciences, New York University \\
719 Broadway, 12th Floor, New York, NY 10003 \\
\texttt{\{xiang, yann\}@cs.nyu.edu} \\
}
\maketitle

\begin{abstract}
This paper shows that simply prescribing ``none of the above'' labels to unlabeled data has a beneficial regularization effect to supervised learning. We call it universum prescription by the fact that the prescribed labels cannot be one of the supervised labels. In spite of its simplicity, universum prescription obtained competitive results in training deep convolutional networks for CIFAR-10, CIFAR-100, STL-10 and ImageNet datasets. A qualitative justification of these approaches using Rademacher complexity is presented. The effect of a regularization parameter -- probability of sampling from unlabeled data -- is also studied empirically.
\end{abstract}

\section{Introduction}

The idea of exploiting the wide abundance of unlabeled data to improve the accuracy of supervised learning tasks is a very natural one. In this paper, we study what is perhaps the simplest way to exploit unlabeled data in the context of deep learning. We assume that the unlabeled samples do not belong to any of the categories of the supervised task, and we force the classifier to produce a ``none of the above'' output for these samples. This is by no means a new idea, but we show empirically and theoretically that doing so has a regularization effect on supervised task and reduces the generalization gap, the expected difference between test and training errors. We study three different ways to prescribe ``none of the above'' outputs, dubbed uniform prescription, dustbin class, and background class and show that they improve the test error of convolutional networks trained on CIFAR-10, CIFAR-100 \cite{K09}, STL-10 \cite{ANL11}, and ImageNet \cite{RDSKSMHKKBBF15}. The method is theoretically justified using Radamacher complexity \cite{BM03}.

Here we briefly describe our three universum prescription methods. Uniform prescription forces a discrete uniform distribution for unlabeled samples. Dustbin class simply adds an extra class and prescribe all unlabeled data to this class. Background class also adds an extra class, but it uses a constant threshold to avoid parameterization.

Our work is a direct extension to learning in the presence of universum \cite{WCSBV06} \cite{CASS07}, originated from \cite{V98} and \cite{V06}. The definition of universum is a set of unlabeled data that are known not to belong to any of the classes but in the same domain. We extended the idea of using universum from support vector machines to deep learning.

Most deep learning approaches utilizing unlabeled data belong to the scope of representation learning (reviewed by \cite{BCV13} and \cite{BL07}) and transfer learning \cite{TP98}. They include ideas like pretraining \cite{EBCMVB10} \cite{HOT06} \cite{RPCL06} and semi-supervised training \cite{RBHVR15} \cite{ZMGL15}. Universum prescription incoporates unlabeled data without imposing priors such as sparsity or reconstruction.

Regularization -- techniques for the control of generalization gap -- has been studied extensively. Most approaches implement a secondary optimization objective, such as an \(L_2\) norm. Some other methods such as dropout \cite{SHKSS14} cheaply simulate model averaging to control the model variance. As part of general statistical learning theory \cite{V95}, \cite{V98}, the justification for regularization is well-developed. We qualitatively justify the methods using Radamacher complexity \cite{BM03}, similar to \cite{WZZLF13}.

\section{Universum Prescription}
\label{sec:pres}

In this section we attempt to formalize the trick of prescribing ``none of the above'' labels -- universum prescription. Consider the problem of exclusive \(k\)-way classification. In learning we hope to find a hypothesis function \(h \in \mathcal{H}\) mapping to \(\mathbb{R}^k\) so that the label is determined by \(y = \argmin_i ~ h_i(x)\). The following assumptions are made.

\begin{enumerate}
\item (Loss assumption) The loss used as the optimization objective is negative log-likelihood:
  \begin{equation}
  L(h, x, y) = h_y (x) + \log \left[ \sum_{i = 1}^{k} \exp(-h_i (x)) \right].
  \end{equation}
\item (Universum assumption) The proportion of unlabeled samples belonging to a supervised class is negligible.
\end{enumerate}

The loss assumption assumes that the probability of class \(y\) given an input \(x\) can be thought of as
\begin{equation}
\Pr[Y = y | x, h] = \frac{\exp(-h_y(x))}{\sum_{i = 1}^{k} \exp(-h_i(x))},
\end{equation}
where \((X, Y) \sim \mathbf{D}\) and \(\mathbf{D}\) is the distribution where labeled data are sampled. We use lowercase letters for values, uppercase letters for random variables and bold uppercase letters for distribution. The loss assumption is simply a necessary detail rather than a limitation, in the sense that one can change the type of loss and use the same principles to derive different universum learning techniques.

The universum assumption implicates that labeled classes are a negligible subset. In many practical cases we only care about a small number of classes, either by problem design or due to high cost in the labeling process. At the same time, a very large amount of unlabeled data is easily obtained. Put in mathematics, assuming we draw unlabeled data from distribution \(\mathbf{U}\), the assumption states that
\begin{equation}
  \label{eq:jtds}
  \Pr_{(X,Y) \sim \mathbf{U}}[X, Y \in \{1, 2, \dots, k\}] \approx 0.
\end{equation}

The universum assumption is opposite to the assumptions of information regularization \cite{CJ06} and transduction learning \cite{CSZ06T} \cite{GVV98}. It has similarities with \cite{ZZ10} that encourages diversity of outputs for ensemble methods. All our methods discussed below prescribe agnostic targets to the unlabeled data. During learning, we randomly present an unlabeled sample to the optimization procedure with probability \(p\).

\subsection{Uniform Prescription}
\label{sect:unif}

It is known that negative log-likelihood is simply a reduced form of cross-entropy
\begin{equation}
L(h, x, y ) = -\sum_{i = 1}^{k} Q[Y = i | x] \log \Pr[Y = i | x, h]
\end{equation}
in which the target probability \(Q[Y = y | x] = 1\) and \(Q[Y = i | x] = 0\) for \(i \neq y\). Under the universum assumption, if we are presented with an unlabeled sample \(x\), we would hope to prescribe some \(Q\) so that every class has some equally minimal probability. \(Q\) also has to satisfy \(\sum_{i = 1}^{k} Q[Y=i|x] = 1\) by the probability axioms. The only possible choice for \(Q\) is then \(Q[Y | x] = 1/k\). The learning algorithm then uses the cross-entropy loss instead of negative log-likelihood.

It is worth noting that uniform output has the maximum entropy among all possible choices. If \(h\) is parameterized as a deep neural network, uniform output is achieved when these parameters are constantly zero. Therefore, uniform prescription may have the effect of reducing the magnitude of parameters, similar to norm-based regularization.

\subsection{Dustbin Class}
\label{sect:dust}

Another way of prescribing agnostic target is to append a ``dustbin'' class to the supervised task. This requires some changes to the hypothesis function \(h\) such that it outputs \(k+1\) targets. For deep learning models one can simply extend the last parameterized layer. All unlabeled data are prescribed to this extra ``dustbin'' class.

The effect of dustbin class is clearly seen in the loss function of an unlabeled sample \((x, k+1)\)
\begin{equation}
L(h, x, k+1) = h_{k+1} (x) + \log \left[ \sum_{i = 1}^{k + 1} \exp(-h_i (x)) \right].
\end{equation}
The second term is a ``soft'' maximum for all dimensions of \(-h\). With an unlabeled sample, the algorithm attempts to introduce smoothness by minimizing probability spikes.

\subsection{Background Class}
\label{sect:bgnd}

We could further simplify dustbin class by removing parameters for class \(k + 1\). For some given threshold constant \(\tau\), we could change the probability of a labeled sample to
\begin{equation}
\Pr[Y = y | x, h] = \frac{\exp(-h_y(x))}{\exp(-\tau) + \sum_{i = 1}^{k} \exp(-h_i(x))},
\end{equation}
and an unlabeled sample
\begin{equation}
\Pr[Y = k + 1 | x, h] = \frac{\exp(-\tau)}{\exp(-\tau) + \sum_{i = 1}^{k} \exp(-h_i(x))}.
\end{equation}

This will result in changes to the loss function of a labeled sample \((x, y)\) as
\begin{equation}
L(h, x, y) = h_y (x) + \log \left[ \exp(-\tau) + \sum_{i = 1}^{k} \exp(-h_i (x)) \right],
\end{equation}
and an unlabeled sample
\begin{equation}
L(h, x, k + 1) = \tau + \log \left[ \exp(-\tau) + \sum_{i = 1}^{k} \exp(-h_i (x)) \right].
\end{equation}

\begin{table}[h]
  \caption{The 21-layer network}
  \label{tab:expi}
  \begin{center}
    \begin{tabular}{ll}
      \multicolumn{1}{c}{\bf LAYERS}  &\multicolumn{1}{c}{\bf DESCRIPTION}
      \\ \hline \\
      1-3           &Conv 256x3x3 \\
      4             &Pool 2x2 \\
      5-8           &Conv 512x3x3 \\
      9             &Pool 2x2 \\
      10-13         &Conv 1024x3x3 \\
      14            &Pool 2x2 \\
      15-18         &Conv 1024x3x3 \\
      19            &Pool 2x2 \\
      20-23         &Conv 2048x3x3 \\
      24            &Pool 2x2 \\
      25-26         &Full 2048 \\
    \end{tabular}
  \end{center}
\end{table}

We call this method background class and \(\tau\) background constant. Similar to dustbin class, the algorithm attempts to minimize the spikes of outputs, but limited to a certain extent by the inclusion of \(\exp(-\tau)\) in the partition function. In our experiments \(\tau\) is always set to 0.

\section{Theoretical Justification}
\label{sec:thry}

In this part, we derive a qualitative justification for universum prescription using probably approximately correct (PAC) learning \cite{V84}. By being ``qualitative'', the justification is in contrast with numerical bounds such as Vapnik-Chervonenkis dimension \cite{VC71} (VC-dim) and others. Our theory is based on Rademacher complexity \cite{BM03}, similar to \cite{WZZLF13} where both dropout \cite{SHKSS14} and dropconnect \cite{WZZLF13} are justified. VC-dim is an upper-bound of Rademacher complexity, suggesting that the latter is more accurate. Previous results on unlabeled data \cite{OAGR11} \cite{OGRA15} assume the same distribution for labeled and unlabeled data, which is impossible under the universum assumption.

\begin{definition}[Empirical Rademacher complexity]
  Let \(\mathcal{F}\) be a family of functions mapping from \( \mathcal{X} \) to \(\mathbb{R}\), and \(S = (x_1, x_2, \dots, x_m)\) a fixed sample of size \(m\) with elements in \(\mathcal{X}\). Then, the empirical Rademacher complexity of \(F\) with respect to the sample \(S\) is defined as:
  \begin{equation}
  \hat{\mathfrak{R}}_S(\mathcal{F}) = \underset{\boldsymbol{\eta}}\E \left[ \sup_{f \in \mathcal{F}} \frac{1}{m} \sum_{i=1}^{m} \eta_i f(x_i) \right]
  \end{equation}
  where \(\boldsymbol{\eta} = (\eta_1, \dots, \eta_m)^T\), with \(\eta_i\)'s independent random variables uniformly distributed on \(\{-1, 1\}\).
\end{definition}

\begin{definition}[Rademacher complexity]
  Let \(\mathbf{D}\) denote the distribution from which the samples were drawn. For any integer \(m \geq 1\), the Rademacher complexity of \(\mathcal{F}\) is the expectation of the empirical Rademacher complexity over all samples of size \(m\) drawn according to \(\mathbf{D}\):
  \begin{equation}
  \mathfrak{R}_m (\mathcal{F}, \mathbf{D}) = \underset{S \sim \mathbf{D}^m}{\E} [\hat{\mathfrak{R}}_S(F)]
  \end{equation}
\end{definition}

Two qualitative properties of Rademacher complexity is worth noting here. First of all, Rademacher complexity is always non-negative by the convexity of supremum
\begin{equation}
\begin{aligned}
  \hat{\mathfrak{R}}_S(\mathcal{F}) & = \underset{\boldsymbol{\eta}}\E \left[ \sup_{f \in \mathcal{F}} \frac{1}{m} \sum_{i=1}^{m} \eta_i f(x_i) \right] \\
  & \geq \sup_{f \in \mathcal{F}} \frac{1}{m} \sum_{i=1}^{m} \underset{\eta_i}{\E} [\eta_i] f(x_i) = 0.
\end{aligned}
\end{equation}
Secondly, if for a fixed input all functions in \(\mathcal{F}\) output the same value, then its Rademacher complexity is 0. Assume for any \(f \in \mathcal{F}\) we have \(f(x) = f_0(x)\), then
\begin{equation}
\begin{aligned}
  \hat{\mathfrak{R}}_S(\mathcal{F}) & = \underset{\boldsymbol{\eta}}\E \left[ \sup_{f \in \mathcal{F}} \frac{1}{m} \sum_{i=1}^{m} \eta_i f(x_i) \right] \\
  & = \underset{\boldsymbol{\eta}}\E \left[ \sup_{f \in \mathcal{F}} \frac{1}{m} \sum_{i=1}^{m} \eta_i f_0(x) \right] \\
  & = \frac{1}{m} \sum_{i=1}^{m} \underset{\eta_i}{\E} [\eta_i] f_0(x) = 0.
\end{aligned}
\end{equation}

Therefore, one way to minimize Rademacher complexity is to regularize functions in \(\mathcal{F}\) such that all functions tend to have the same output for a given input. Universum prescription precisely does that -- the prescribed outputs for unlabeled data are all constantly the same.

The principal PAC-learning result is a bound for functions that are finite in outputs. We use the formulation by \cite{Z13}, but anterior results exist \cite{BBL02} \cite{BM03} \cite{K01} \cite{KP00}.
\begin{theorem} [Approximation bound with finite bound on output]
  For an energy function \cite{LCHRH06} \(\mathcal{E}(h,x,y)\) over hypothesis class \(\mathcal{H}\), input set \(\mathcal{X}\) and output set \(\mathcal{Y}\), if it has lower bound 0 and upper bound \(M > 0\), then with probability at least \(1-\delta\), the following holds for all \(h \in \mathcal{H} \):
  \begin{equation}
    \label{eq:pcbd}
    \begin{aligned}
    & \underset{(x,y) \sim \mathbf{D}}{\E}[\mathcal{E}(h,x,y)] \leq \\
    & \frac{1}{m} \sum_{(x,y) \in S} \mathcal{E}(h,x,y) + 2 \mathfrak{R}_m(\mathcal{F}, \mathbf{D}) + M\sqrt{\frac{\log{\frac{2}{\delta}}}{2m}},
    \end{aligned}
  \end{equation}
  where the function family \(\mathcal{F}\) is defined as
  \begin{equation}
    \mathcal{F} = \left\{ \mathcal{E}(h,x,y) | h \in \mathcal{H} \right\}.
  \end{equation}
  \(\mathbf{D}\) is the distribution for \((x,y)\), and \(S\) is a sample set of size \(m\) drawn indentically and independently from \(\mathbf{D}\).
\end{theorem}

The meaning of the theorem is two-fold. When applying the theorem to the joint problem of training using both labeled and unlabeled data, the third term on the right hand of inequality \ref{eq:pcbd} is reduced by the augmentation of the extra data. The joint Rademacher complexity is written as \(\mathfrak{R}_m(\mathcal{F}, (1-p)\mathbf{D} + p \mathbf{U})\), which is reduced when we prescribe constant outputs to unlabeled data.

The second fold is that when the theorem applies to the supervised distribution \(\mathbf{D}\), we would hope that \(\mathfrak{R}_n(\mathcal{F}, \mathbf{D})\) can be bounded by \(\mathfrak{R}_m(\mathcal{F}, (1-p)\mathbf{D} + p \mathbf{U})\), where \(n\) is the number of supervised samples randomly chosen by the joint problem. Note that the number \(n\) follows a binomial distribution with mean \((1-p)m\). Such a bound can be achieved in a probable and approximate sense.

\begin{theorem} [Rademacher complexity bound on distribution mixture]
  \label{thm:rbdm}
  Assume we have a joint problem where \(p \leq 0.5\) and there are \(m\) random training samples from the joint distribution \((1-p)\mathbf{D} + p \mathbf{U}\). With probability at least \(1-\delta\), the following holds
  \begin{equation}
    \label{eq:rbdm}
    \begin{aligned}
    & \mathfrak{R}_n(\mathcal{F}, \mathbf{D}) \leq \\
    & \frac{2-p}{(1-p)\left(1-p - \sqrt{\frac{\log(1/\delta)}{2m}}\right)} \mathfrak{R}_m(\mathcal{F}, (1-p)\mathbf{D} + p \mathbf{U}),
    \end{aligned}
  \end{equation}
  where \(n\) is a random number indicating the number of supervised samples in the total joint samples, and \(m\) is large enough such that
  \begin{equation}
  1-p-\sqrt{\frac{\log(1/\delta)}{2m}} > 0.
  \end{equation}
\end{theorem}

We present the proof of theorem \ref{thm:rbdm} in the supplemental material, which utilizes Hoeffding's inequality \cite{H63} \cite{S74}. The theorem tells us that the Rademacher complexity of the supervised problem can be bounded by that of the joint problem. The universum prescription algorithm attempts to make the Rademacher complexity of the joint problem small. Therefore, universum prescription improves generalization by incorporating unlabeled data.


However, theorem \ref{thm:rbdm} has a requirement that \( p \leq 0.5\), otherwise the bound is not achievable. Also, the value of \((2-p)/(1-p)^2\) -- the asymptotic constant factor in inequality \ref{eq:rbdm} when \(m\) is large -- is monitonally increasing with respect to \(p\) with a range of \([2, 6]\) when \(p \leq 0.5\). These facts indicate that we need to keep \(p\) small. The following sections show that there is improvement if \(p\) is small, but training and testing errors became worse when \(p\) is large.

\begin{table}[h]
  \caption{The 17-layer network}
  \label{tab:expn}
  \begin{center}
    \begin{tabular}{ll}
      \multicolumn{1}{c}{\bf LAYERS}  &\multicolumn{1}{c}{\bf DESCRIPTION}
      \\ \hline \\
      1-3           &Conv 128x3x3 \\
      4             &Pool 2x2 \\
      5-7           &Conv 256x3x3 \\
      8             &Pool 2x2 \\
      9-11          &Conv 512x3x3 \\
      12            &Pool 2x2 \\
      13-15         &Conv 1024x3x3 \\
      16            &Pool 2x2 \\
      17-19         &Conv 2048x3x3 \\
      20            &Pool 2x2 \\
      21-22         &Full 4096 \\
    \end{tabular}
  \end{center}
\end{table}

Finally, in terms of numerical asymptotics, theorem \ref{thm:rbdm} suggests that \(\mathfrak{R}_n(\mathcal{F}, \mathbf{D}) \leq \mathbf{O}(1/\sqrt{m})\), instead of the commonly known result \(\mathfrak{R}_n(\mathcal{F}, \mathbf{D}) \leq \mathbf{O}(1/\sqrt{n})\). This bounds the supervised problem with a tighter asymptotical factor because there are more joint samples than supervised samples.

\section{Experiments on Image Classification}
\label{sec:expi}

In this section we test the methods on some image classification tasks. Three series of datasets -- CIFAR-10/100 \cite{K09}, STL-10 \cite{ANL11} and ImageNet \cite{RDSKSMHKKBBF15} -- are chosen due to the availability of unlabeled data. For CIFAR-10/100 and STL-10 datasets, we used a 21-layer convolutional network (ConvNet) \cite{LBDHHHJ89} \cite{LBBH98}, in which the inputs are 32-by-32 images and all convolutional layers are 3-by-3 and fully padded. For ImageNet, the model is a 17-layer ConvNet with 64-by-64 images as inputs. These models are inspired by \cite{SZ14}, in which all pooling layers are max-pooling, and ReLUs \cite{NH10} are used as the non-linearity. Two dropout \cite{SHKSS14} layers of probability 0.5 are inserted before the final two linear layers.

\begin{table}[h]
  \caption{Result for baseline and uniform prescription. The numbers are percentages.}
  \label{tab:expd}
  \begin{center}
    \addtolength{\tabcolsep}{-3pt}
    \begin{tabular}{lrrrrrr}
      \textbf{DATASET}  & \multicolumn{3}{c}{\textbf{BASELINE}} & \multicolumn{3}{c}{\textbf{UNIFORM}} \\
      & \multicolumn{1}{c}{\small\textbf{Train}} & \multicolumn{1}{c}{\small\textbf{Test}} & \multicolumn{1}{c}{\small\textbf{Gap}} & \multicolumn{1}{c}{\small\textbf{Train}} & \multicolumn{1}{c}{\small\textbf{Test}} & \multicolumn{1}{c}{\small\textbf{Gap}} \\
      \hline \\
      CIFAR-10 & \textbf{0.00} & 7.02 & 7.02 & 0.72 & 7.59 & 6.87 \\
      CIFAR-100 F. & \textbf{0.09} & 37.58 & 37.49 & 4.91 & 36.23 & 31.32 \\
      CIFAR-100 C. & \textbf{0.04} & 22.74 & 22.70 & 0.67 & 23.42 & 22.45 \\
      STL-10 & \textbf{0.00} & \textbf{31.16} & 31.16 & 2.02 & 36.54 & 34.52 \\
      STL-10 Tiny & \textbf{0.00} & 31.16 & 31.16 & 0.62 & 30.15 & 29.47 \\
      ImageNet-1 & \textbf{10.19} & 34.39 & 24.20 & 13.84 & 34.61 & 20.77 \\
      ImageNet-5 & \textbf{1.62} & 13.68 & 12.06 & 3.02 & 13.70 & 10.68 \\
    \end{tabular}
    \addtolength{\tabcolsep}{4pt}
    \vspace{-10pt}
  \end{center}
\end{table}

\begin{table}[h]
  \caption{Result for dustbin class and background class. Continuation of table \ref{tab:expd}}
  \label{tab:expr}
  \begin{center}
    \addtolength{\tabcolsep}{-3pt}
    \begin{tabular}{lrrrrrr}
      \textbf{DATASET}  & \multicolumn{3}{c}{\textbf{DUSTBIN}} & \multicolumn{3}{c}{\textbf{BACKGROUND}} \\
      & \multicolumn{1}{c}{\small\textbf{Train}} & \multicolumn{1}{c}{\small\textbf{Test}} & \multicolumn{1}{c}{\small\textbf{Gap}} & \multicolumn{1}{c}{\small\textbf{Train}} & \multicolumn{1}{c}{\small\textbf{Test}} & \multicolumn{1}{c}{\small\textbf{Gap}} \\
      \hline \\
      CIFAR-10 & 0.07 & \textbf{6.66} & \textbf{6.59} & 1.35 & 8.38 & 7.03 \\
      CIFAR-100 F.  & 2.52 & \textbf{32.84} & \textbf{30.32} & 8.56 & 40.57 & 42.01 \\
      CIFAR-100 C.  & 0.40 & \textbf{20.45} & \textbf{20.05} & 3.73 & 24.97 & 21.24 \\
      STL-10 & 3.03 & 36.58 & 33.55 & 14.89 & 38.95 & \textbf{24.06} \\
      STL-10 Tiny  & 0.00 & \textbf{27.96} & \textbf{27.96} & 0.11 & 30.38 & 30.27 \\
      ImageNet-1 & 13.80 & \textbf{33.67} & \textbf{19.87} & 13.43 & 34.69 & 21.26 \\
      ImageNet-5 & 2.83 & \textbf{13.35} & \textbf{10.52} & 2.74 & 13.84 & 11.10 \\
    \end{tabular}
    \addtolength{\tabcolsep}{4pt}
    \vspace{-10pt}
  \end{center}
\end{table}

The algorithm used is stochastic gradient descent with momentum \cite{P64} \cite{SMDH13} 0.9 and a minibatch size of 32. The initial learning rate is 0.005 which is halved every 60,000 minibatch steps for CIFAR-10/100 and every 600,000 minibatch steps for ImageNet. The training stops at 400,000 steps for CIFAR-10/100 and STL10, and 2,500,000 steps for ImageNet. Table \ref{tab:expi} and \ref{tab:expn} summarize the configurations. The weights are initialized in the same way as \cite{HZRS15}. The following data augmentation steps are used.

\begin{enumerate}
  \setlength{\itemsep}{2pt}
\item (Horizontal flip.) Flip the image horizontally with probability 0.5.
\item (Scale.) Randomly scale the image between \(1/1.2\) and \(1.2\) times of its height and width.
\item (Crop.) Randomly crop a 32-by-32 (or 64-by-64 for ImageNet) region in the scaled image.
\end{enumerate}

\subsection{CIFAR-10 and CIFAR-100}

The samples of CIFAR-10 and CIFAR-100 datasets \cite{K09} are from the 80 million tiny images dataset \cite{TFF08}. Each dataset contains 60,000 samples, consitituting a very small portion of 80 million. This is an ideal case for our methods, in which we can use the entire 80 million images as the unlabeled data. The CIFAR-10 dataset has 10 classes, and CIFAR-100 has 20 (coarse) or 100 (fine-grained) classes.

Table \ref{tab:expd} and \ref{tab:expr} contain the results. The three numbers in each tabular indicate training error, testing error and generalization gap. Bold numbers are the best ones for each case. The generalization gap is approximated by the difference between testing and training errors. All the models use unlabeled data with probability \(p=0.2\).

We compared other single-model results on CIFAR-10 and CIFAR-100 (fine-grained case) in table \ref{tab:cfcp}. It shows that our network is competitive to the state of the art. Although \cite{G14} has the best results, we believe that by applying out universum prescription methods to their model design could also improve the results further.

\begin{table}[h]
  \caption{Comparison of single-model CIFAR-10 and CIFAR-100 results, in second and third columns. The fourth column indicates whether data augmentation is used for CIFAR-10. The numbers are percentages.}
  \label{tab:cfcp}
  \begin{center}
    \begin{tabular}{lccc}
      \textbf{REF.} & \textbf{10} & \textbf{100} & \textbf{AUG.}
      \\ \hline \\
      \cite{G14} & 6.28 & 24.30 & YES \\
      (ours) & 6.66 & 32.84 & YES \\
      \cite{LXPZT15} & 7.97 & 34.57 & YES \\
      \cite{LCY13} & 8.81 & 35.68 & YES \\
      \cite{GWMCB13} & 9.38 & 38.57 & YES \\
      \cite{WZZLF13} & 11.10 & N/A & NO \\
      \cite{MR13} & 15.13 & 42.51 & NO \\
    \end{tabular}
  \end{center}
\end{table}

\subsection{STL-10}

The STL-10 dataset \cite{ANL11} has size 96-by-96 for each image. We downsampled them to 32-by-32 in order to use the same model. The dataset contains a very small number of training samples -- 5000. The accompanying unlabeled data contain 100,000 samples. There is no guarantee that these unlabeled samples do not blong to the supervised classes \cite{ANL11}, therefore universum prescription failed. To verify that the extra data is the problem, an experiment using the 80 million tiny images as the unlabeled dataset is shown in table \ref{tab:expd} and \ref{tab:expr}. In this case the improvement is observed. Due to long training times of our models, we did not perform 10-fold training in the original paper \cite{ANL11}.

One interesting observation is that the results on STL-10 became better with the use of 80 million tiny images instead of the original extra data. It indicates that dataset size and whether universum assumption is satisfied are affecting factors for the effectiveness of universum prescription.

\subsection{ImageNet}

The ImageNet dataset \cite{RDSKSMHKKBBF15} for classification task has in total 1,281,167 training images and 50,000 validation images. The reported testing errors are evaluated on this validation dataset. During training, we resize images to minimum dimension 64, and then feed a random 64-by-64 crop to the network. Same test-time augmentation technique as in \cite{SLJSRAEVR15} are applied, with size variants \{64, 72, 80, 88\}, where each image is viewed in 144 crops.

The extra data comes from the large ImageNet 2011 release\footnote{\texttt{http://www.image-net.org/releases}}, for which we only keep the classes whomself and whose children do not belong to the supervised classes. This is enabled by the super-subordinate (is-a) relation information provided with the WordNet distribution \cite{M95} because all ImageNet classes are nouns of WordNet. Both top-1 and top-5 results are reported in tables \ref{tab:expd} and \ref{tab:expr}.

In all experiments dustbin class provides best results. We believe that it is because the extra class is parameterized, which makes it adapt better on the unlabeled samples.

\begin{table}[h]
  \caption{ConvNet for the study of \(p\)}
  \label{tab:para}
  \begin{center}
    \begin{tabular}{ll}
      \multicolumn{1}{c}{\bf LAYERS}  &\multicolumn{1}{c}{\bf DESCRIPTION}
      \\ \hline \\
      1             &Conv 1024x5x5 \\
      2             &Pool 2x2 \\
      3             &Conv 1024x5x5 \\
      4-7           &Conv 1024x3x3 \\
      8             &Pool 2x2 \\
      9-11          &Full 2048 \\
    \end{tabular}
  \end{center}
\end{table}

\section{Effect of the Regularization Parameter}
\label{sec:para}

It is natural to ask how would the change of the probability \(p\) of sampling from unlabeled data affect the results. In this section we show the experiments. To prevent an exhaustive search on the regularization parameter from overfitting our models on the testing data, we use a different model for this section. It is described in table \ref{tab:para}, which has 9 parameterized layers in total. The design is inspired by \cite{SEZMFL13}. For each choice of \(p\) we conducted 6 experiments combining universum prescription models and dropout. The dropout layers are two ones added in between the fully-connected layers with dropout probability 0.5. Figure \ref{fig:para} shows the results.

From figure \ref{fig:para} we can conclude that increasing \(p\) will descrease generalization gap. However, we cannot make \(p\) too large since after a certain point the training collapses and both training and testing errors become worse. This confirms the assumptions and conclusions from theorem \ref{thm:rbdm}.

\begin{figure}[h]
  \begin{tabular}{ccc}
    Train & Test& Gap \\
    \includegraphics[width=0.3\columnwidth]{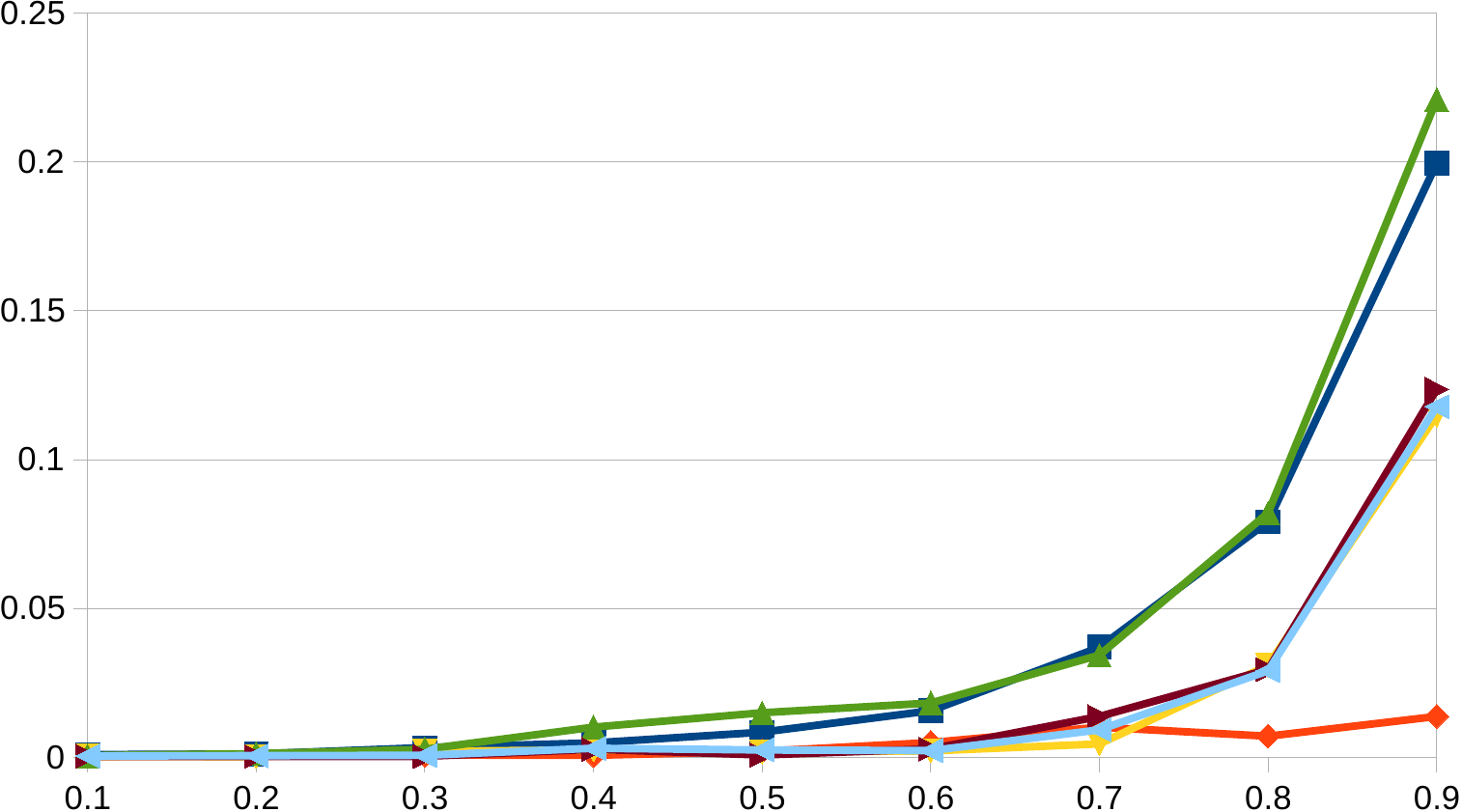} & \includegraphics[width=0.3\columnwidth]{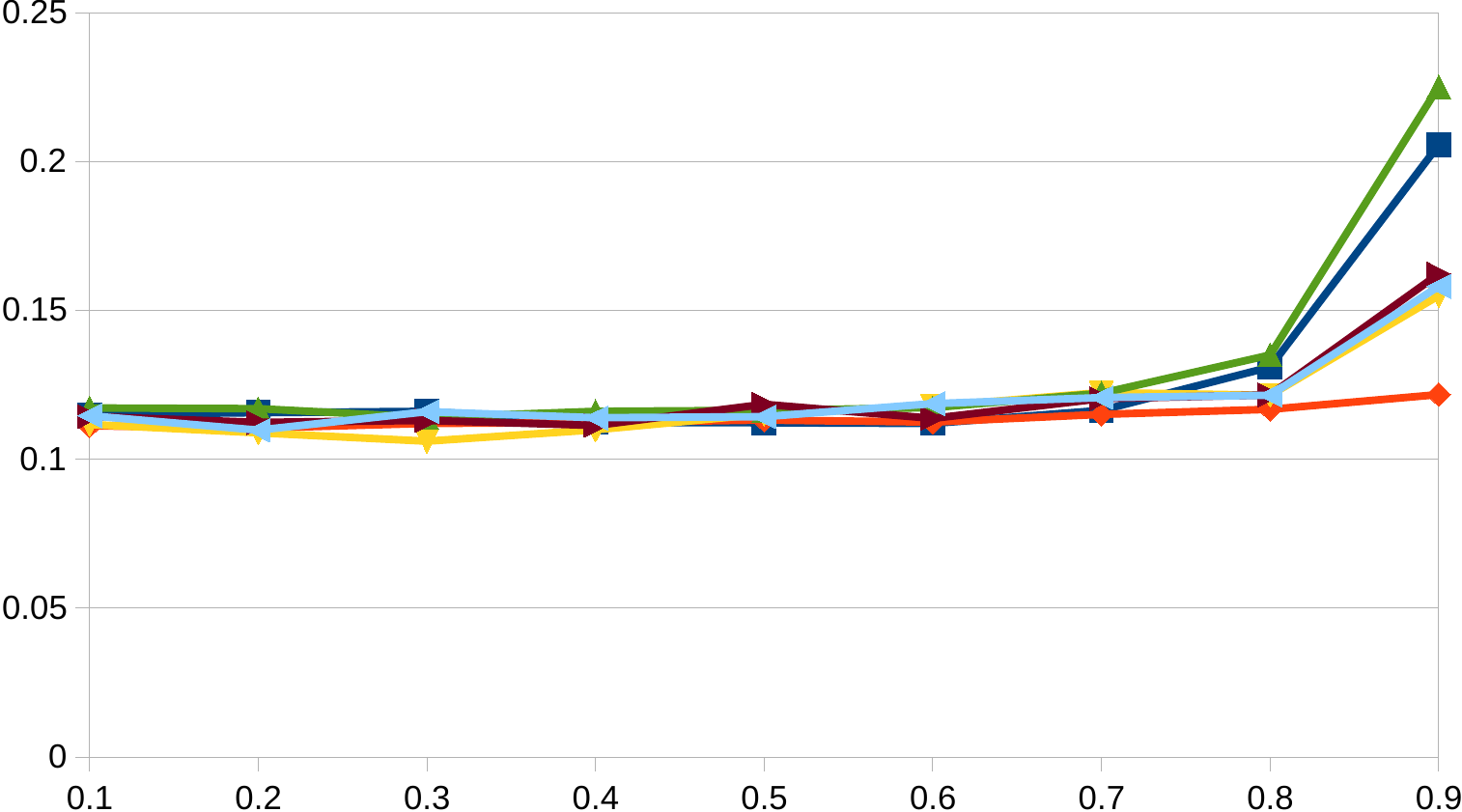} & \includegraphics[width=0.3\columnwidth]{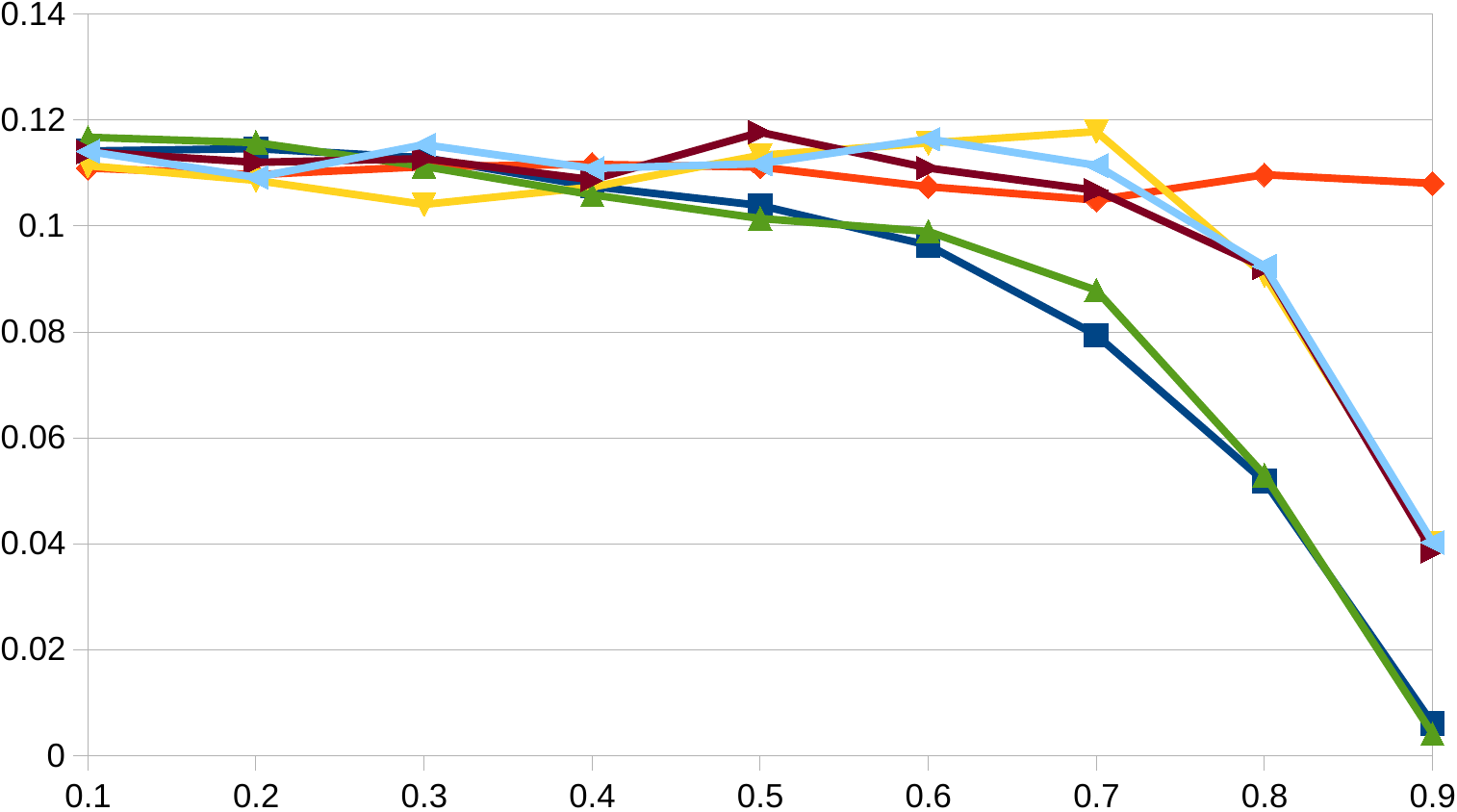} \\
    \includegraphics[width=0.3\columnwidth]{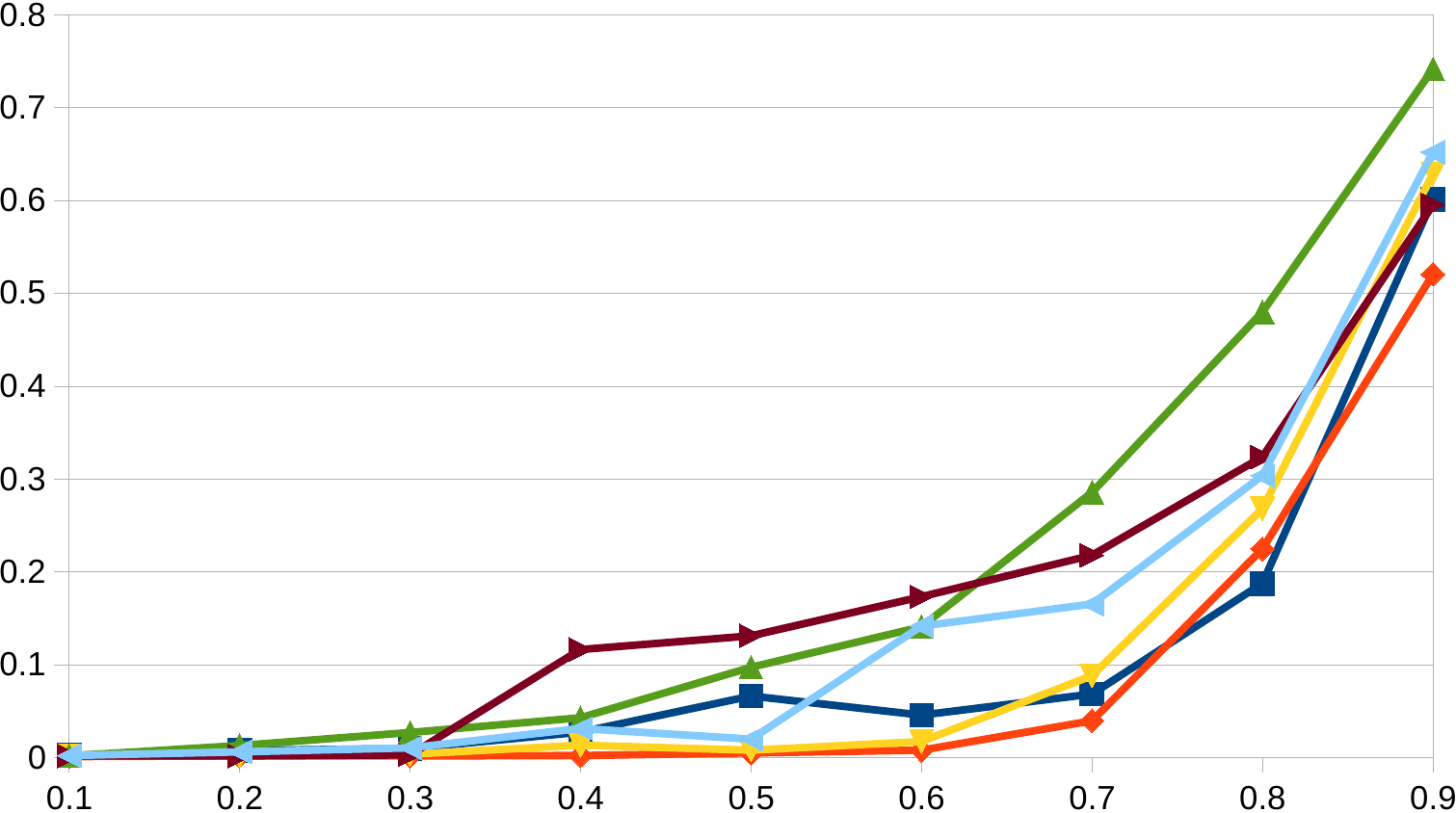} & \includegraphics[width=0.3\columnwidth]{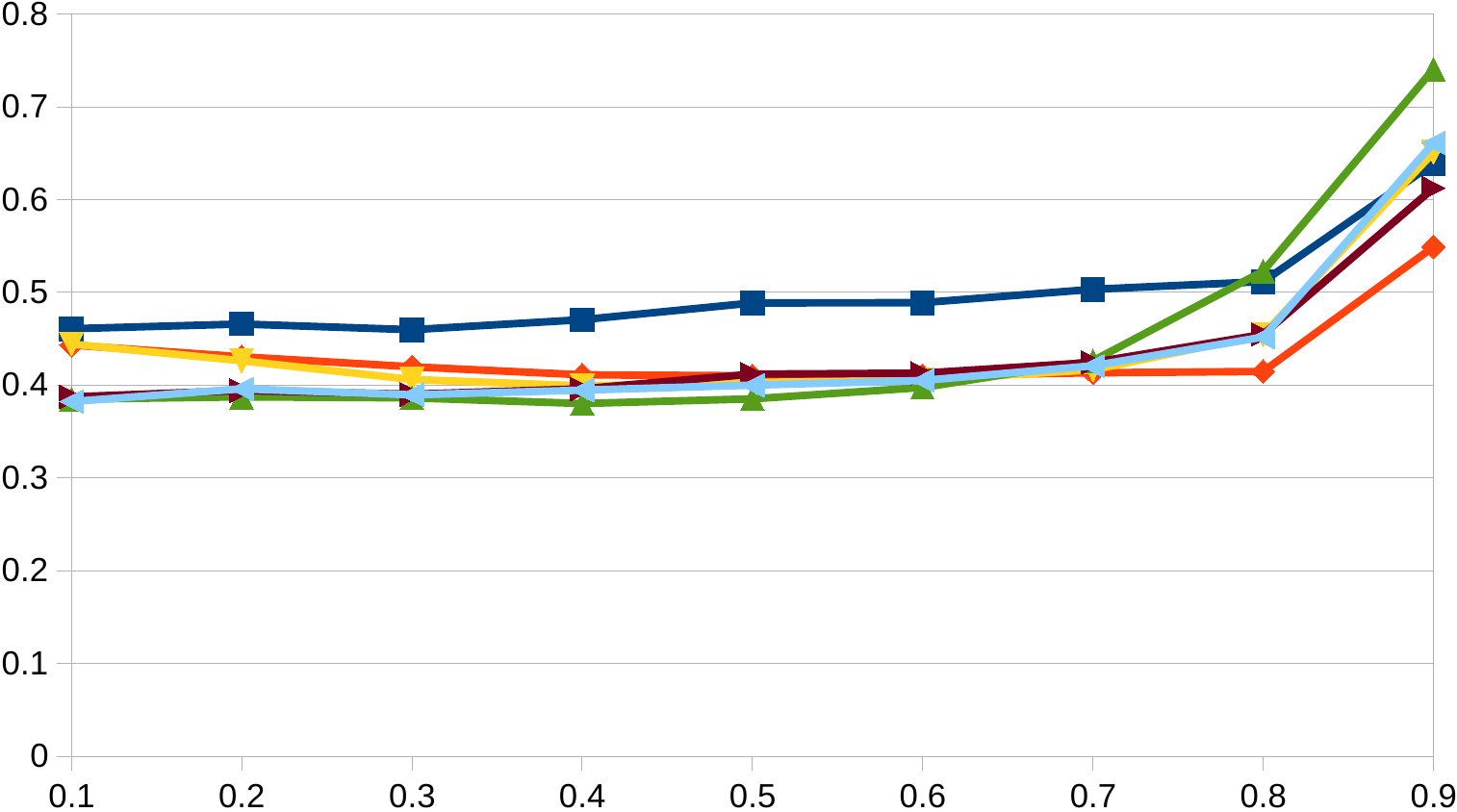} & \includegraphics[width=0.3\columnwidth]{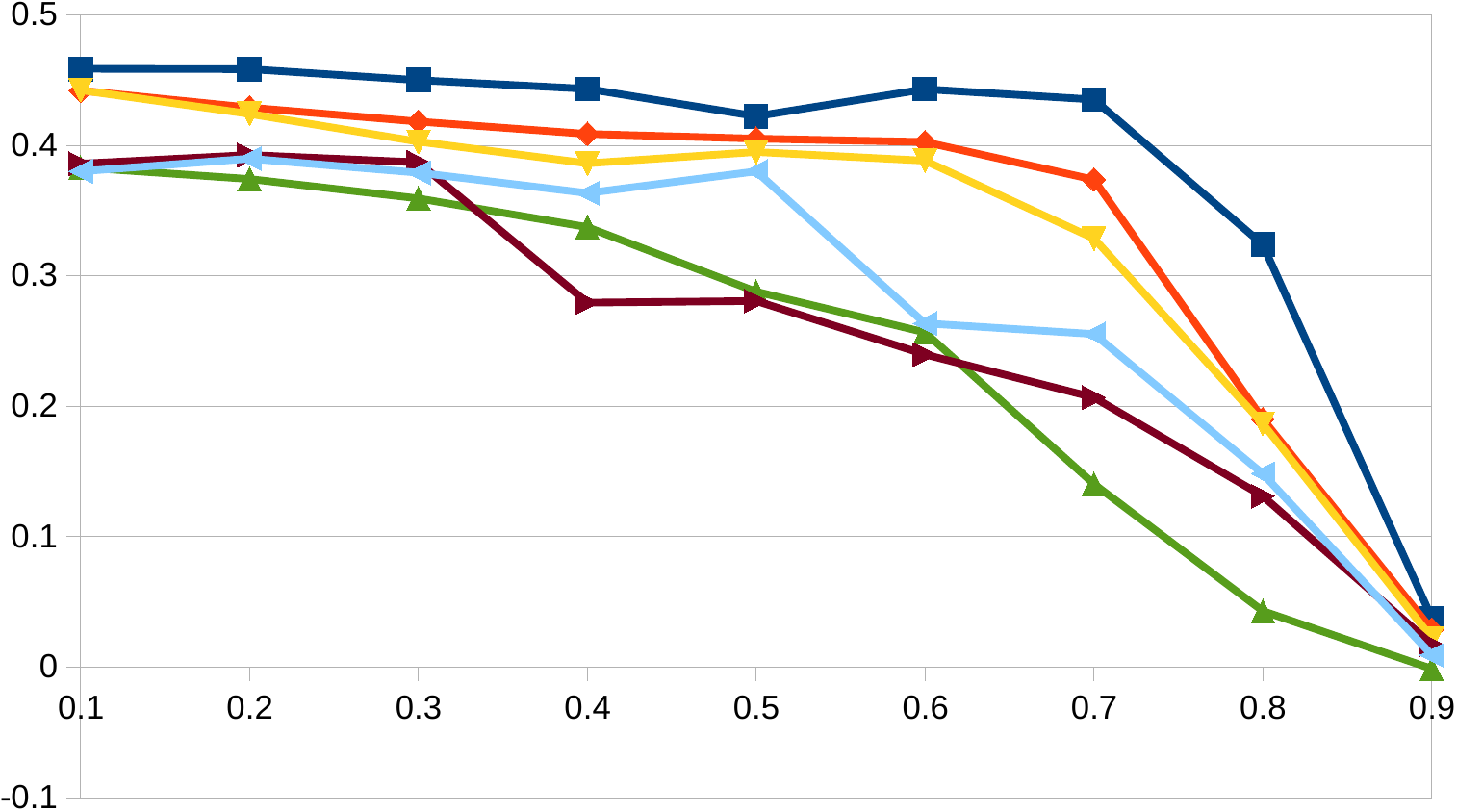} \\
    \includegraphics[width=0.3\columnwidth]{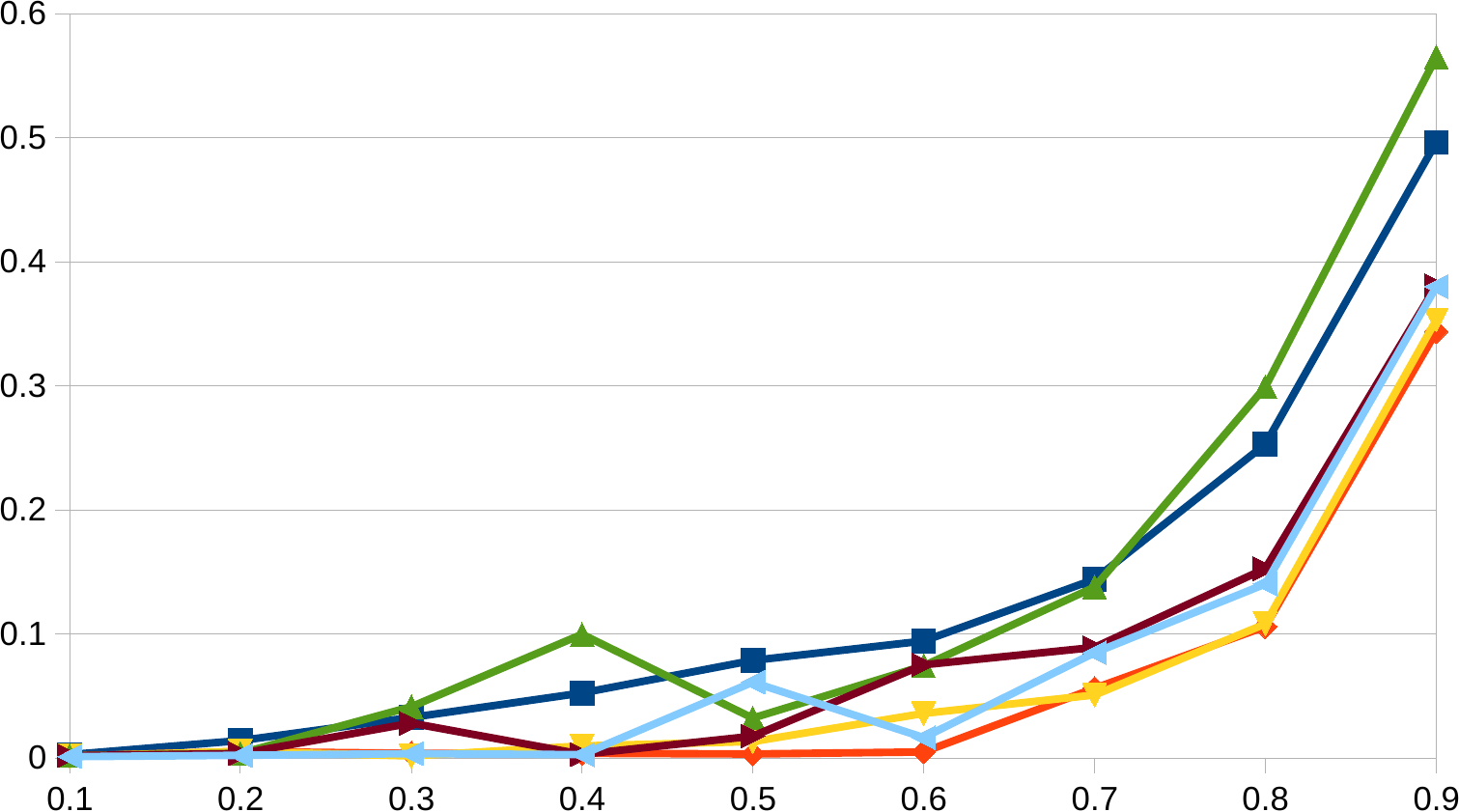} & \includegraphics[width=0.3\columnwidth]{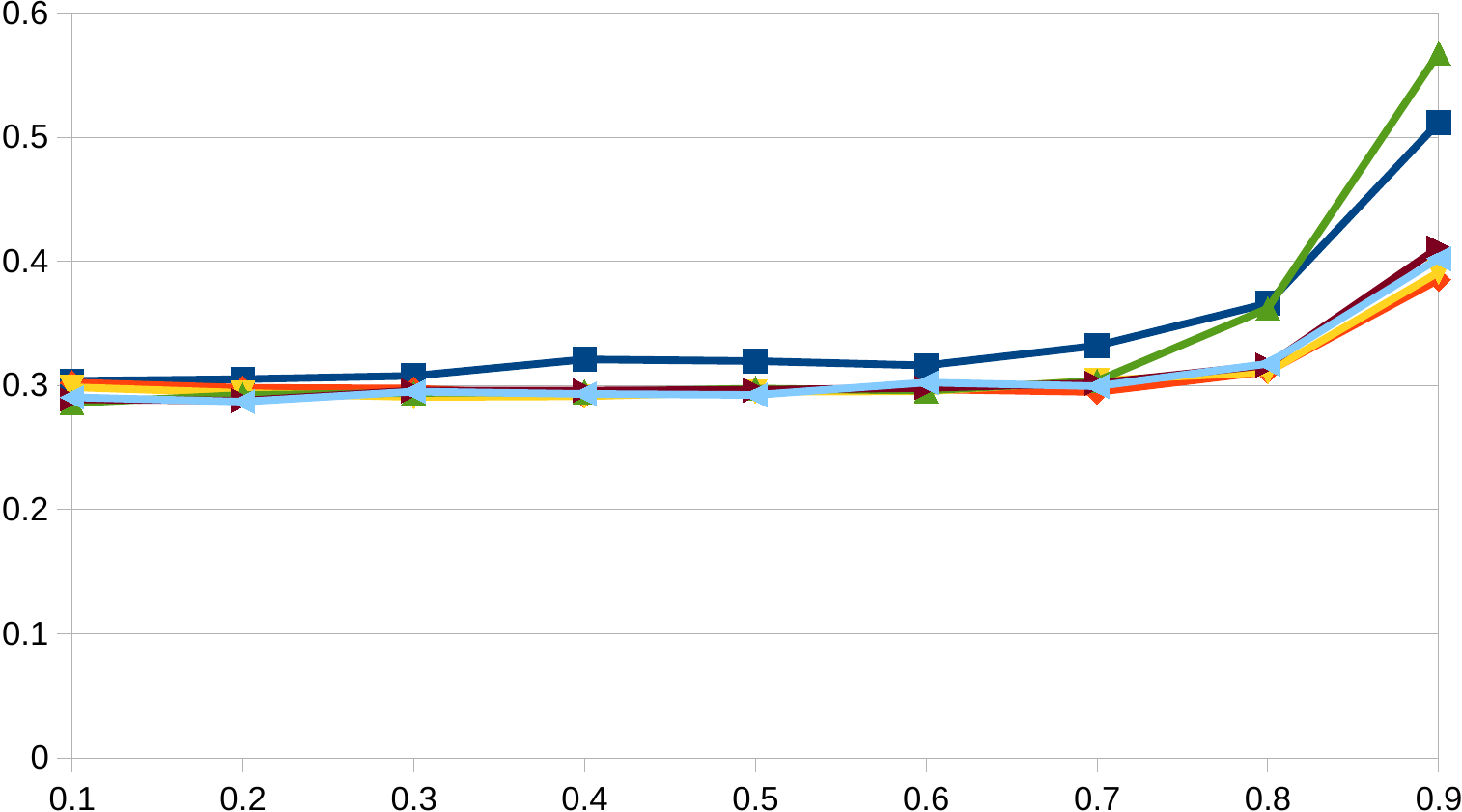} & \includegraphics[width=0.3\columnwidth]{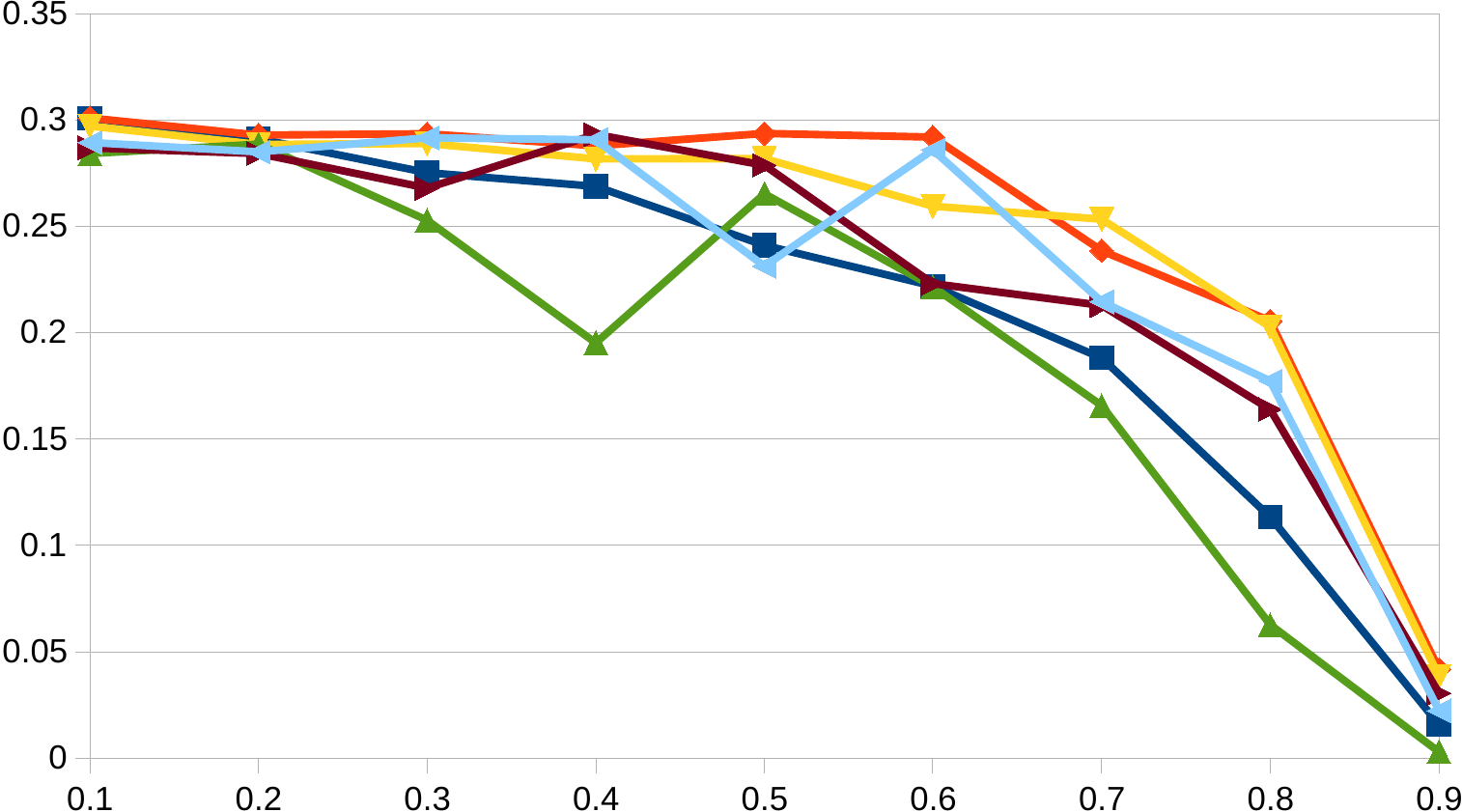} \\
    \includegraphics[width=0.3\columnwidth]{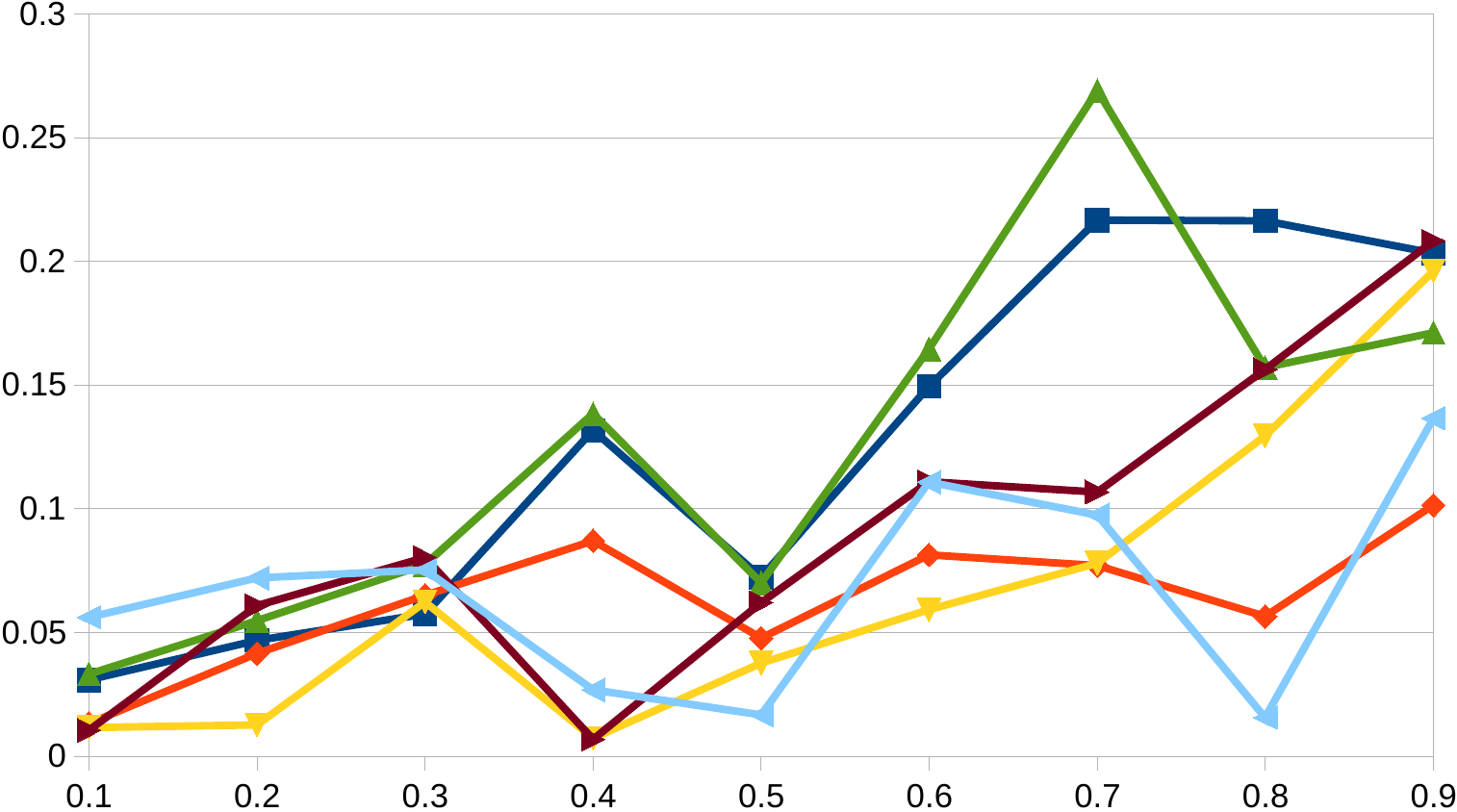} & \includegraphics[width=0.3\columnwidth]{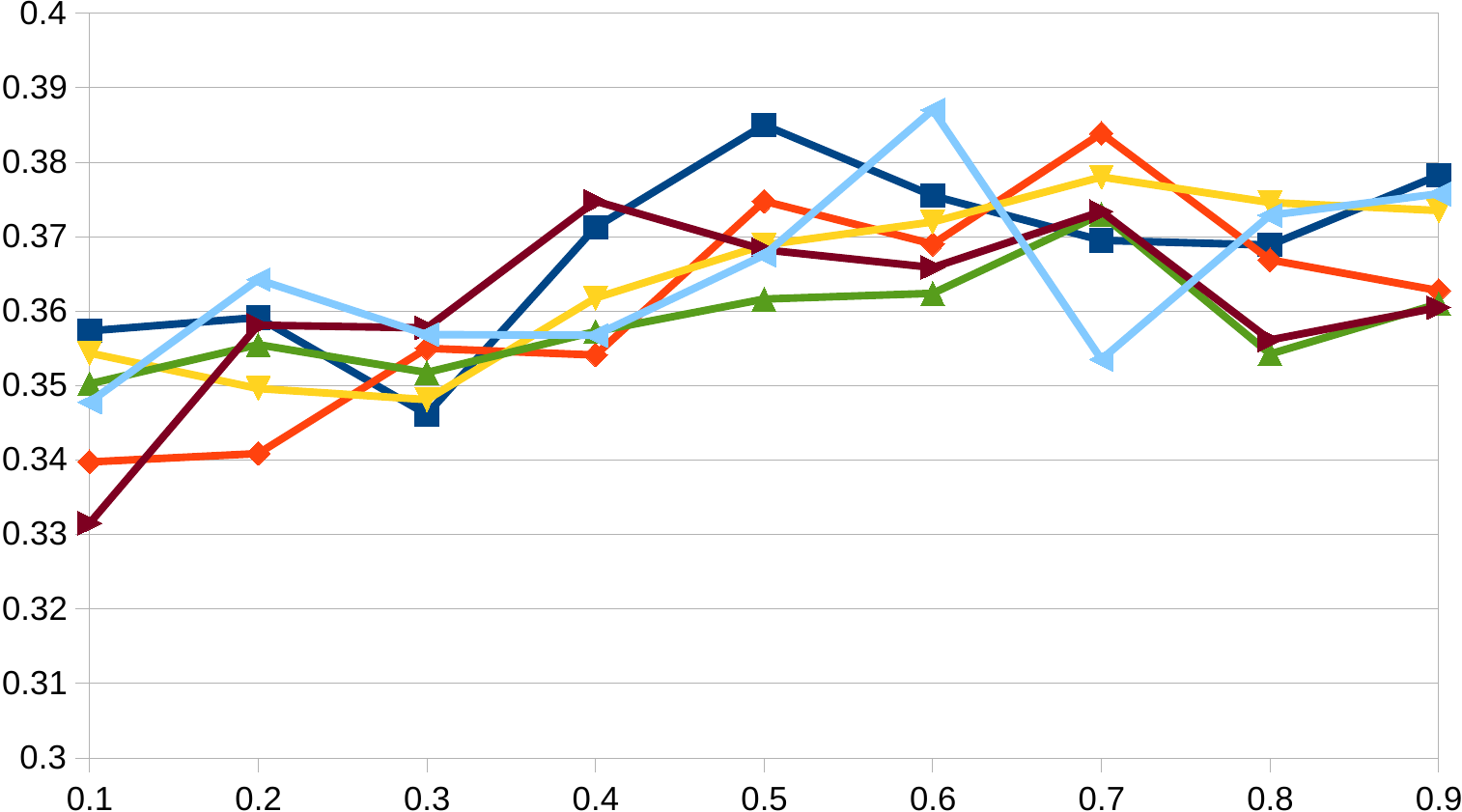} & \includegraphics[width=0.3\columnwidth]{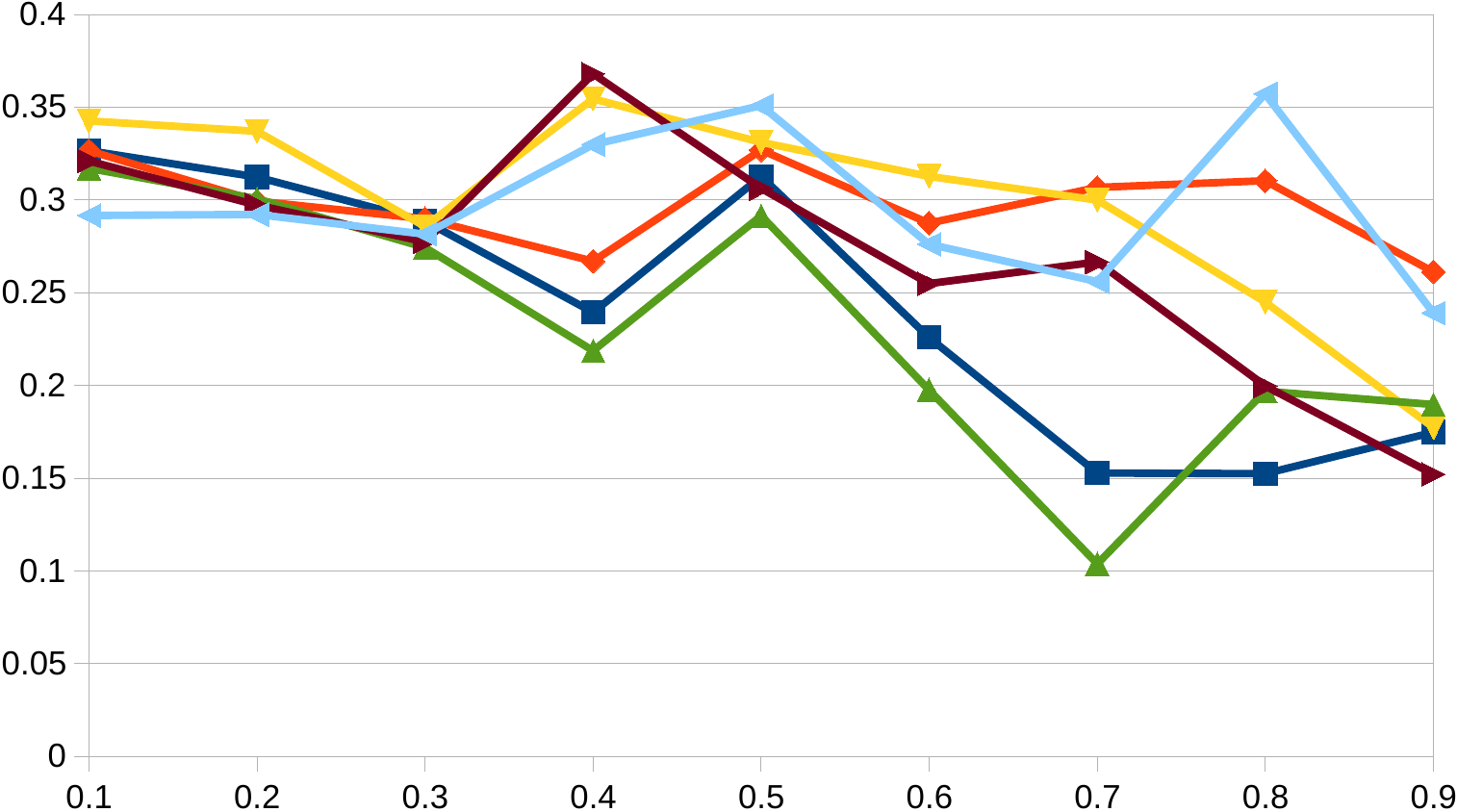} \\
    \multicolumn{3}{c}{\includegraphics[width=0.7\columnwidth]{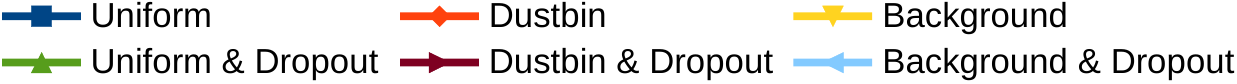}}
  \end{tabular}

  \caption{Experiments on regularization parameter. The four rows are CIFAR-10, CIFAR-100 fine-grained, CIFAR-100 coarse and STL-10 respectively.}
  \label{fig:para}
\end{figure}

Comparing between CIFAR-10/100 and STL-10, one conclusion is that that the model variance is affected by the combined size of labeled and unlabeled datasets. The variance on training and testing errors are extremely small on CIFAR-10/100 datasets because the extra data we used is almost unlimited (in total 80 million), but on STL-10 the variance seems to be large with much smaller combined size of training and extra datasets. This suggests that using universum prescription with a large abundance of extra data could improve the stability of supervised learning algorithms.

Finally, the comparison between using and not using dropout does not show a difference. This suggests that the regularization effect of universum prescription alone is comparable to that of dropout.

\section{Conclusion and Outlook}

This article shows that universum prescription can be used to regularize a multi-class classification problem using extra unlabeled data. Two assumptions are made. One is that loss used is negative log-likelihood and the other is negligible probability of a supervised sample existing in the unlabeled data. The loss assumption is a necessary detail rather than a limitation. The three universum prescription methods are uniform prescription, dustbin class and background class.

We further provided a theoretical justification. Theorem \ref{thm:rbdm} suggests that asymptotically the generalization ability of the supervised problem could be bounded by the joint problem, which has more samples due to the addition of unlabeled data. Experiments are done using CIFAR-10, CIFAR-100, STL-10 and ImageNet datasets. The effect of the regularization parameter is also studied empirically.

These experiments show that all three universum prescrition methods provide certain improvement over the generalization gap, whereas dustbin class constantly performs the best because the parameterized extra class can adapt better to the unlabeled samples. Further conclusions include that additional unlabeled data can improve the variance of models during training, and that the results are comparable to data-agnostic regularization using dropout.

In the future, we hope to apply these methods to a broader range of problems.

\section*{Acknowledgments}

We gratefully acknowledge NVIDIA Corporation with the donation of 2 Tesla K40 GPUs used for this research. Sainbayar Sukhbaatar offered many useful comments. Aditya Ramesh and Junbo Zhao helped cross-checking the proofs.

\bibliography{article}
\bibliographystyle{aaai}

\vfill
\pagebreak

\section*{Supplemental: proof of theorem \ref{thm:rbdm}}

This supplemental material shares the bibliography of the main paper. As an outline of our proof, we first establish a relation between \(\mathfrak{R}_m (\mathcal{F}, \mathbf{D})\) and \(\mathfrak{R}_m (\mathcal{F}, (1-p)\mathbf{D} + p \mathbf{U})\), and then another relation between \(\mathfrak{R}_n (\mathcal{F}, \mathbf{D})\) and \(\mathfrak{R}_m (\mathcal{F}, \mathbf{D})\). The first part requires the following lemmas.

\begin{lemma}[Separation of dataset on empirical Rademacher complexity]
  \label{lem:sepr}
  Let \(S\) be a dataset of size \(m\). If \(S_1\) and \(S_2\) are two non-overlap subset of \(S\) such that \(|S_1| = m - i\), \(|S_2| = i\) and \(S_1 \cup S_2 = S\), then the following two inequalities hold
  \begin{equation}
    \label{eq:sepr}
    \hat{\mathfrak{R}}_S (\mathcal{F}) \leq \frac{m - i}{m} \hat{\mathfrak{R}}_{S_1} (\mathcal{F}) + \frac{i}{m} \hat{\mathfrak{R}}_{S_2} (\mathcal{F}).
  \end{equation}
\end{lemma}
\begin{proof}
  Let \((x_j, y_j) \in S_1\) for \(j = 1, 2, \dots, m - i\) and \((x_j, y_j) \in S_2\) for \(i = m - j + 1, m - j + 2, \dots, m\). Denote \(\mathbf{N}\) as the discrete uniform distribution on \(\{1, -1\}\). We can derive by the convexity of supremum and symmetry of \(\mathbf{N}\)
  \[
  \begin{aligned}
    & \hat{\mathfrak{R}}_{S} (\mathcal{F}) = \underset{\boldsymbol{\eta} \sim \mathbf{N}^{m}}\E \left[ \sup_{f \in \mathcal{F}} \frac{1}{m} \sum_{j=1}^{m} \eta_j f(x_j) \right] \\
    = &\frac{2}{m} \underset{\boldsymbol{\eta} \sim \mathbf{N}^{m}}\E \left[ \sup_{f \in \mathcal{F}} \left( \frac{1}{2} \sum_{j=1}^{m - i} \eta_j f(x_j) + \frac{1}{2} \sum_{j=m-i+1}^{m} \eta_j f(x_j) \right) \right] \\
    \leq & \frac{2}{m} \underset{\boldsymbol{\eta} \sim \mathbf{N}^{m}}\E \left[ \frac{1}{2} \sup_{f \in \mathcal{F}} \left( \sum_{j=1}^{m - i} \eta_j f(x_j) \right) + \frac{1}{2} \sup_{f \in \mathcal{F}} \left( \sum_{j=m-i+1}^{m} \eta_j f(x_j) \right) \right] \\
    = & \frac{m - i}{m} \underset{\boldsymbol{\eta} \sim \mathbf{N}^{m - i}}\E \left[ \sup_{f \in \mathcal{F}} \frac{1}{m-i} \sum_{j=1}^{m - i} \eta_j f(x_j) \right] + \\
    & \quad \frac{i}{m} \underset{\boldsymbol{\eta} \sim \mathbf{N}^{i}}\E \left[ \sup_{f \in \mathcal{F}} \frac{1}{i} \sum_{j=m-i+1}^{m} \eta_j f(x_j) \right] \\
    = &\frac{m - i}{m} \hat{\mathfrak{R}}_{S_1} (\mathcal{F}) + \frac{i}{m} \hat{\mathfrak{R}}_{S_2} (\mathcal{F}).
  \end{aligned}
  \]
\end{proof}

\begin{lemma}[Sample size inequality for Rademacher complexity]
  \label{lem:sinr}
  Assume \(0 \leq n \leq m\). If \(|S_n| = n\), \(|S_m| = m\) and \(S_m = S_n \cup \{x_{n+1}, x_{n+2}, \dots, x_{m}\}\), then
  \begin{equation}
    \label{eq:sine}
    n \hat{\mathfrak{R}}_{S_n} (\mathcal{F}) \leq m \hat{\mathfrak{R}}_{S_m} (\mathcal{F}),
  \end{equation}
  and
  \begin{equation}
    \label{eq:sint}
    n \mathfrak{R}_n (\mathcal{F}, \mathbf{D}) \leq m \mathfrak{R}_m (\mathcal{F}, \mathbf{D}).
  \end{equation}
\end{lemma}
\begin{proof}
  First of all, it is obvious that inequality \ref{eq:sint} can be established using mathematical induction if we have \(m \mathfrak{R}_m (\mathcal{F}, \mathbf{D}) \leq (m + 1) \mathfrak{R}_{m + 1} (\mathcal{F}, \mathbf{D})\) for all \(m \geq 0\). To prove this, we first establish that if \(S_m = \{x_1, x_2, \dots, x_m\}\) and \(S_{m + 1} = \{x_1, x_2, \dots, x_m, x_{m+1}\}\) (i.e., \(S_{m+1} = S_m \cup \{x_{m + 1}\}\)), then \(m \hat{\mathfrak{R}}_{S_m} (\mathcal{F}) \leq (m+1) \hat{\mathfrak{R}}_{S_{m+1}} (\mathcal{F})\), which can also establish inequality \ref{eq:sine}.

  For any \( \boldsymbol{\eta}_m =\{\eta_1, \eta_2, \dots, \eta_m\} \) and \( \boldsymbol{\eta}_{m + 1} =\{\eta_1, \eta_2, \dots, \eta_m, \eta_{m+1}\}\), that is, \( \boldsymbol{\eta}_{m + 1} = \boldsymbol{\eta}_m \cup \{\eta_{m + 1}\}\), let \( f_0 = \argmax_{f \in \mathcal{F}} \sum_{i = 1}^{m} \eta_i f(x_i)\). By definition of supremum, we have
  \[
  \begin{aligned}
  \sup_{f \in \mathcal{F}} \sum_{i = 1}^{m + 1} \eta_i f(x_i) & \geq \sum_{i = 1}^{m + 1} \eta_i f_0 (x_i) \\
  & = \sum_{i = 1}^{m} \eta_i f_0 (x_i) + \eta_{m+1}f_0 (x_{m+1}) \\
  & = \sup_{f \in \mathcal{F}} \sum_{i = 1}^{m} \eta_i f(x_i) + \eta_{m+1} f_0(x_{m+1}).
  \end{aligned}
  \]
  Taking espectation over \(\boldsymbol{\eta}_{m+1}\), by the symmetry of distribution \(\mathbf{N}\), we obtain
  \[
  \begin{aligned}
    & \underset{\boldsymbol{\eta}_{m+1} \sim \mathbf{N}^{m+1}}\E \left[ \sup_{f \in \mathcal{F}} \sum_{i = 1}^{m + 1} \eta_i f(x_i) \right] \\
    & \geq \underset{\boldsymbol{\eta}_{m+1} \sim \mathbf{N}^{m}}\E \left[ \sup_{f \in \mathcal{F}} \sum_{i = 1}^{m} \eta_i f(x_i) + \eta_{m+1} f_0(x_{m+1})  \right] \\
    & = \underset{\boldsymbol{\eta}_{m} \sim \mathbf{N}^{m}}\E \left[ \sup_{f \in \mathcal{F}} \sum_{i = 1}^{m} \eta_i f(x_i) \right] + \underset{\eta_{m + 1} \sim \mathbf{N}}\E \left[ \eta_{m+1} \right] f_0(x_{m+1}) \\
    & = \underset{\boldsymbol{\eta}_{m} \sim \mathbf{N}^{m}}\E \left[ \sup_{f \in \mathcal{F}} \sum_{i = 1}^{m} \eta_i f(x_i) \right].
  \end{aligned}
  \]
  By the definition of \(\hat{\mathfrak{R}}_{S_m} (\mathcal{F})\), the inequality above implies \(m \hat{\mathfrak{R}}_{S_m} (\mathcal{F}) \leq (m+1) \hat{\mathfrak{R}}_{S_{m+1}} (\mathcal{F})\). Then, by taking espectation over \(S_{m+1}\) we can obtain
  \[
  \begin{aligned}
  (m + 1) \mathfrak{R}_{m + 1} (\mathcal{F}, \mathbf{D}) & = \underset{S_{m+1} \sim \mathbf{D}^{m+1}} \E \left[ (m+1) \hat{\mathfrak{R}}_{S_{m+1}} \right] \\
  & \geq \underset{S_{m} \sim \mathbf{D}^{m}} \E \left[ m \hat{\mathfrak{R}}_{S_{m}} \right] = m \mathfrak{R}_{m} (\mathcal{F}, \mathbf{D}).
  \end{aligned}
  \]
  The lemma can therefore be easily established by mathematical induction.
\end{proof}

Using the lemmas above, the relation between \(\mathfrak{R}_m (\mathcal{F}, \mathbf{D})\) and \(\mathfrak{R}_m (\mathcal{F}, (1-p)\mathbf{D} + p \mathbf{U})\) can be established as the following theorem, by assuming \(p \leq 0.5\).

\begin{theorem}[Relation of Rademacher complexities in distribution mixture]
  \label{thm:rerm}
  If \(p \leq 0.5\), then
  \begin{equation}
    \label{eq:rerm}
    \mathfrak{R}_m (\mathcal{F}, \mathbf{D}) \leq \frac{2 - p}{1 - p} \mathfrak{R}_m (\mathcal{F}, (1-p)\mathbf{D} + p \mathbf{U}).
  \end{equation}
\end{theorem}

\begin{proof}
  For any function space \(\mathcal{F}\) and distribution \(\mathbf{D}\), denote \(\mathfrak{R}_0(\mathcal{F}, \mathbf{D}) = 0\) and \(\hat{\mathfrak{R}}_{\emptyset}(\mathcal{F}) = 0\). By definition of Rademacher complexity and lemma \ref{lem:sepr}, we get
  \[
  \begin{aligned}
    & \mathfrak{R}_m(\mathcal{F}, \mathbf{D})  = \mathfrak{R}_m(\mathcal{F}, (1-p)\mathbf{D} + p \mathbf{D}) \\
    & = \underset{S \sim ((1-p)\mathbf{D} + p \mathbf{D})^m}{\E} \left [\hat{\mathfrak{R}}_S(\mathcal{F}) \right] \\
    & = \sum_{i=0}^m \binom{m}{i} (1-p)^i p^{m-i} \underset{S_1 \sim \mathbf{D}^{i}}{\E} \left[ \underset{S_2 \sim \mathbf{D}^{m-i}}{\E} \left[\hat{\mathfrak{R}}_{S_1 \cup S_2}(\mathcal{F}) \right] \right] \\
    & \leq \sum_{i=0}^m \binom{m}{i} (1-p)^{i} p^{m-i} \\
    & \quad \cdot \underset{S_1 \sim \mathbf{D}^{i}}{\E} \left[ \underset{S_2 \sim \mathbf{D}^{m-i}}{\E} \left [ \frac{i}{m} \hat{\mathfrak{R}}_{S_1} (\mathcal{F}) + \frac{m-i}{m} \hat{\mathfrak{R}}_{S_2} (\mathcal{F}) \right] \right] \\
    & = \sum_{i=0}^m \binom{m}{i} (1-p)^{i} p^{m-i} \\
    & \quad \cdot \left( \underset{S_1 \sim \mathbf{D}^{i}}{\E} \left[ \frac{i}{m} \hat{\mathfrak{R}}_{S_1} (\mathcal{F}) \right] + \underset{S_2 \sim \mathbf{D}^{m-i}}{\E} \left [ \frac{m-i}{m} \hat{\mathfrak{R}}_{S_2} (\mathcal{F}) \right] \right) \\
    & = \sum_{i=0}^m \binom{m}{i} (1-p)^{i} p^{m-i} \left[ \frac{i}{m} \mathfrak{R}_{i} (\mathcal{F}, \mathbf{D}) + \frac{m-i}{m} \mathfrak{R}_{i} (\mathcal{F}, \mathbf{D})\right] \\
    & = \left[\sum_{i=0}^m \binom{m}{i} (1-p)^{i} p^{m-i} \frac{i}{m} \mathfrak{R}_{i} (\mathcal{F}, \mathbf{D}) \right] \\
    & \quad + \left[ \sum_{i=0}^m \binom{m}{i} (1-p)^{m-i} p^i \frac{i}{m} \mathfrak{R}_{i} (\mathcal{F}, \mathbf{D})\right] \\
    & = \left[\sum_{i=0}^m \binom{m}{i} (1-p)^{i} p^{m-i} \frac{i}{m} \mathfrak{R}_{i} (\mathcal{F}, \mathbf{D}) \right] \\
    & \quad + \left[ \sum_{i=0}^{\lfloor m / 2 \rfloor} \binom{m}{i} (1-p)^{m-i} p^i \frac{i}{m} \mathfrak{R}_{i} (\mathcal{F}, \mathbf{D})\right] \\
    & \quad + \left[ \sum_{i=\lfloor m / 2 \rfloor + 1}^{m} \binom{m}{i} (1-p)^{m-i} p^i \frac{i}{m} \mathfrak{R}_{i} (\mathcal{F}, \mathbf{D})\right]. \\
  \end{aligned}
  \]
  The proof proceeds by handling the three parts on the right-hand side of the inequality above separately.

  For the first part, using lemma \ref{lem:sinr}, we can get
  \[
  \begin{aligned}
    & \mathfrak{R}_m(\mathcal{F}, (1-p)\mathbf{D} + p \mathbf{U}) \\
    & = \sum_{i=0}^m \binom{m}{i} (1-p)^{i} p^{m-i} \underset{S_1 \sim \mathbf{D}^{i}}{\E} \left[ \underset{S_2 \sim \mathbf{U}^{m-i}}{\E} \left[ \hat{\mathfrak{R}}_{S_1 \cup S_2}(\mathcal{F}) \right ] \right] \\
    & \geq \sum_{i=0}^m \binom{m}{i} (1-p)^{i} p^{m - i} \underset{S_1 \sim \mathbf{D}^{i}}{\E} \left[ \underset{S_2 \sim \mathbf{U}^{m-i}}{\E} \left [\frac{i}{m}\hat{\mathfrak{R}}_{S_1}(\mathcal{F}) \right] \right] \\
    & = \sum_{i=0}^m \binom{m}{i} (1-p)^{i} p^{m - i} \underset{S_1 \sim \mathbf{D}^{i}}{\E} \left[\frac{i}{m}\hat{\mathfrak{R}}_{S_1}(\mathcal{F}) \right] \\
    & = \sum_{i=0}^m \binom{m}{i} (1-p)^{i} p^{m - i} \frac{i}{m} \mathfrak{R}_{i}(\mathcal{F}, \mathbf{D}). \\
  \end{aligned}
  \]

  The second part can also proceed using lemma \ref{lem:sinr}. It is essentially upper-bounded by the first part. By the fact that \(i \leq m - i \) for \(0 \leq i \leq \lfloor m / 2 \rfloor\), we obtain
  \[
  \begin{aligned}
    & \sum_{i=0}^{\lfloor m / 2 \rfloor} \binom{m}{i} (1-p)^{m-i} p^i \frac{i}{m} \mathfrak{R}_{i} (\mathcal{F}, \mathbf{D}) \\
    & \leq \sum_{i=0}^{\lfloor m / 2 \rfloor} \binom{m}{i} (1-p)^{m-i} p^i \frac{m - i}{m} \mathfrak{R}_{m-i} (\mathcal{F}, \mathbf{D}) \\
    & = \sum_{i=m-\lfloor m / 2 \rfloor}^{m} \binom{m}{i} (1-p)^{i} p^{m-i} \frac{i}{m} \mathfrak{R}_{i} (\mathcal{F}, \mathbf{D}) \\
    & \leq \sum_{i=0}^{m} \binom{m}{i} (1-p)^{i} p^{m-i} \frac{i}{m} \mathfrak{R}_{i} (\mathcal{F}, \mathbf{D}) \\
    & \leq \mathfrak{R}_m(\mathcal{F}, (1-p)\mathbf{D} + p \mathbf{U}) \\
  \end{aligned}
  \]

  The third part takes advantage of the assumption that \( p \leq 0.5\). We know that for \(\lfloor m / 2 \rfloor + 1 \leq i \leq m\), the assumption \(p \leq 0.5\) implies
  \[
  (1-p)^{m-i} p^i \leq \frac{p}{1-p} (1-p)^{i} p^{m-i}.
  \]
  Therefore, using the first part, we achieve
  \[
  \begin{aligned}
    & \sum_{i=\lfloor m / 2 \rfloor + 1}^{m} \binom{m}{i} (1-p)^{m-i} p^i \frac{i}{m} \mathfrak{R}_{i} (\mathcal{F}, \mathbf{D}) \\
    & \leq \sum_{i=\lfloor m / 2 \rfloor + 1}^{m} \binom{m}{i} \frac{p}{1-p} (1-p)^{i} p^{m-i} \frac{i}{m} \mathfrak{R}_{i} (\mathcal{F}, \mathbf{D}) \\
    & =  \frac{p}{1-p} \sum_{i=\lfloor m / 2 \rfloor + 1}^{m} \binom{m}{i} (1-p)^{i} p^{m-i} \frac{i}{m} \mathfrak{R}_{i} (\mathcal{F}, \mathbf{D}) \\
    & \leq  \frac{p}{1-p} \sum_{i=0}^{m} \binom{m}{i} (1-p)^{i} p^{m-i} \frac{i}{m} \mathfrak{R}_{i} (\mathcal{F}, \mathbf{D}) \\
    & \leq  \frac{p}{1-p} \mathfrak{R}_m(\mathcal{F}, (1-p)\mathbf{D} + p \mathbf{U}).\\
  \end{aligned}
  \]

  By combining all the three parts above, we establish
  \[
  \begin{aligned}
    & \mathfrak{R}_m(\mathcal{F}, \mathbf{D}) \\
    & \leq \left[\sum_{i=0}^m \binom{m}{i} (1-p)^{i} p^{m-i} \frac{i}{m} \mathfrak{R}_{i} (\mathcal{F}, \mathbf{D}) \right] \\
    & \quad + \left[ \sum_{i=0}^{\lfloor m / 2 \rfloor} \binom{m}{i} (1-p)^{m-i} p^i \frac{i}{m} \mathfrak{R}_{i} (\mathcal{F}, \mathbf{D})\right] \\
    & \quad + \left[ \sum_{i=\lfloor m / 2 \rfloor + 1}^{m} \binom{m}{i} (1-p)^{m-i} p^i \frac{i}{m} \mathfrak{R}_{i} (\mathcal{F}, \mathbf{D})\right] \\
    & \leq \left( 1 + 1 + \frac{p}{1-p} \right) \mathfrak{R}_m(\mathcal{F}, (1-p)\mathbf{D} + p \mathbf{U})\\
    & = \frac{2 - p}{1-p}  \mathfrak{R}_m(\mathcal{F}, (1-p)\mathbf{D} + p \mathbf{U}).\\
  \end{aligned}
  \]
  The proof for theorem \ref{thm:rerm} is therefore concluded.
\end{proof}

The relation between \(\mathfrak{R}_n (\mathcal{F}, \mathbf{D})\) and \(\mathfrak{R}_m (\mathcal{F}, \mathbf{D})\) is achieved by the following theorem.

\begin{theorem}[Concentration inequality of subset Rademacher complexity]
  \label{thm:subr}
  Assume in solving the joint problem we obtained \(m\) idependently and identically distributed samples. Let the random number \(n\) represent the number of supervised sample obtained among these \(m\) joint samples with a proprtion probability of \(1-p\). Then, with probability at least \(1-\delta\), the following holds
  \begin{equation}
    \label{eq:subr}
    \mathfrak{R}_n (\mathcal{F}, \mathbf{D}) \leq \frac{\mathfrak{R}_m (\mathcal{F}, \mathbf{D})}{1-p-\sqrt{\frac{\log(1/\delta)}{2m}}},
  \end{equation}
  for large enough \(m\) such that
  \begin{equation}
  1-p-\sqrt{\frac{\log(1/\delta)}{2m}} > 0.
  \end{equation}
\end{theorem}

\begin{proof}
  Using lemma \ref{lem:sinr}, we only need to prove an upper bound for \(m/n\). Since we know that \(n\) follows a binomial distribution with mean \((1-p)m\), using Hoeffding's inequality \cite{H63} \cite{S74}, we can obtain
  \[
  \Pr\left[n \leq (1-p-\epsilon) m\right] \leq \exp(-2 \epsilon^2 m),
  \]
  or put differently,
  \[
  \Pr\left[\frac{m}{n} \leq \frac{1}{1-p-\epsilon} \right] \geq 1 - \exp(-2 \epsilon^2 m).
  \]
  The inequality is obtained by setting \(\delta = \exp(-2 \epsilon^2 m)\). The proof assumes that \(m\) is large enough such that
  \[
  1-p-\sqrt{\frac{\log(1/\delta)}{2m}} > 0.
  \]
\end{proof}

As a result, theorem \ref{thm:rbdm} can be obtained by directly combining theorem \ref{thm:rerm} and theorem \ref{thm:subr}.

\end{document}